\documentclass{article} 
\usepackage[preprint]{colm2026_conference}
\usepackage{multirow} 
\usepackage{microtype}
\usepackage{hyperref}
\usepackage{url}
\usepackage{booktabs}
\usepackage{amsmath}

\usepackage{amssymb}
\usepackage{subcaption}
\usepackage{graphicx}
\newtheorem{theorem}{Theorem}

\newenvironment{proof}[1][Proof]{\noindent\textbf{#1.} }{\hfill$\square$\par\medskip}

\newtheorem{remark}[theorem]{Remark}

\usepackage{lineno}

\definecolor{darkblue}{rgb}{0, 0, 0.5}
\hypersetup{colorlinks=true, citecolor=darkblue, linkcolor=darkblue, urlcolor=darkblue}

\title{Adaptive Stopping for Multi-Turn LLM Reasoning}
\newcommand{\squishlist}{
  \begin{list}{$\bullet$}{
    \setlength{\itemsep}{0pt}
    \setlength{\parsep}{3pt}
    \setlength{\topsep}{3pt}
    \setlength{\partopsep}{0pt}
    \setlength{\leftmargin}{1.5em}
    \setlength{\labelwidth}{1em}
    \setlength{\labelsep}{0.5em}
  }
}

\newcommand{\squishlisttwo}{
  \begin{list}{$\bullet$}{
    \setlength{\itemsep}{0pt}
    \setlength{\parsep}{0pt}
    \setlength{\topsep}{0pt}
    \setlength{\partopsep}{0pt}
    \setlength{\leftmargin}{2em}
    \setlength{\labelwidth}{1.5em}
    \setlength{\labelsep}{0.5em}
  }
}

\newcommand{\squishend}{
  \end{list}
}

\author{Xiaofan Zhou \\
Department of Computer Science\\
University of Illinois Chicago\\
Chicago, IL 60607, USA \\
\texttt{xzhou77@uic.edu} \\
\And
Huy Nguyen \\
Data science \\
Augustana College \\
Rock Island, IL 61201, USA \\
\texttt{huynguyen23@augustana.edu} \\
\AND
Bo Yu \\
Department of Civil and Environmental Engineering \\
University of Utah \\
Salt Lake City, UT 84112, USA \\
\texttt{bo.yu@utah.edu}
\AND
Chenxi Liu \\
Department of Civil and Environmental Engineering \\
University of Utah \\
Salt Lake City, UT 84112, USA \\
\texttt{chenxi.liu@utah.edu}
\AND
Lu Cheng \\
Department of Computer Science\\
University of Illinois Chicago\\
Chicago, IL 60607, USA \\
\texttt{lucheng@uic.edu}
}

\begin{document}

\ifcolmsubmission
\linenumbers
\fi

\maketitle

\begin{abstract}
Large Language Models (LLMs) increasingly rely on multi-turn reasoning and interaction, such as adaptive retrieval-augmented generation (RAG) and ReAct-style agents, to answer difficult questions. These methods improve accuracy by iteratively retrieving information, reasoning, or acting, but introduce a key challenge: \textbf{When should the model stop?} Existing approaches rely on heuristic stopping rules or fixed turn budgets and provide no formal guarantees that the final prediction still contains the correct answer. This limitation is particularly problematic in high-stakes domains such as finance and healthcare, where unnecessary turns increase cost and latency, while stopping too early risks incorrect decisions.
Conformal prediction (CP) provides formal coverage guarantees, but existing LLM-CP methods only apply to a single model output and cannot handle multi-turn pipelines with adaptive stopping. To address this gap, we propose \textbf{M}ult\textbf{i}-Turn Language Models with \textbf{C}onformal \textbf{P}rediction (MiCP), the first CP framework for multi-turn reasoning. MiCP allocates different error budgets across turns, enabling the model to stop early while maintaining an overall coverage guarantee. We demonstrate MiCP on adaptive RAG and ReAct, where it achieves the target coverage on both single-hop and multi-hop question answering benchmarks while reducing the number of turns, inference cost, and prediction set size. We further introduce a new metric that jointly evaluates coverage validity and answering efficiency.
\end{abstract}

\section{Introduction}
Recent Large Language Models (LLMs) have become one of the most transformative technologies in scientific research \citep{rane2023contribution}, industry \citep{wang2025large}, and daily life assistance \citep{achiam2023gpt}. Beyond producing a direct answer in a single forward pass, modern LLM systems increasingly rely on multi-turn interaction and reasoning. When the initial context is insufficient, the model may iteratively retrieve additional evidence, generate intermediate reasoning steps, or interact with external tools before deciding on a final answer. Representative examples include adaptive Retrieval-Augmented Generation (RAG) \citep{moskvoretskii2025adaptive,DBLP:conf/aaai/LiTXCZY26}, which repeatedly retrieves documents until enough evidence has been gathered, and ReAct-style agents \citep{yao2022react}, which interleave reasoning and actions in a multi-turn loop.

These multi-turn strategies create a new challenge: deciding when to stop. If the model stops too early, it may fail to gather enough information and return an incorrect answer. If it continues for too many turns, the additional retrievals or reasoning steps increase latency, computation cost, and sometimes even introduce more errors. Existing approaches therefore, rely on heuristic stopping rules, such as confidence thresholds, fixed retrieval budgets, or manually designed criteria. However, these heuristics provide no formal guarantee that the resulting prediction still contains the correct answer.

The need for principled uncertainty quantification becomes especially important in high-stakes applications such as finance and healthcare. Conformal prediction (CP) \citep{angelopoulos2021gentle,zhou2025conformal} has recently emerged as a powerful framework to provide finite-sample coverage guarantees for LLM outputs. By calibrating a nonconformity score on held-out data, CP constructs a prediction set that contains the correct answer with a user-specified confidence. Existing LLM-CP methods have successfully been applied to short-answer tasks \citep{su-etal-2024-api,li-etal-2024-traq} and long-form generation \citep{mohri2024language}. Nevertheless, all prior approaches assume a single-shot prediction: the model produces one final answer, and CP is applied only after the entire process is complete.

In contrast, multi-turn LLM pipelines may terminate after different numbers of turns depending on the difficulty of the question. A conformal predictor for such systems must therefore answer two questions simultaneously: (1) whether the model should stop at the current turn, and (2) whether the resulting prediction set still satisfies the desired coverage guarantee. Naively applying standard CP independently at each turn fails to solve this problem. First, it does not provide a guarantee on the overall coverage across all possible stopping times. Second, the sampling cost grows combinatorially across turns. For example, if three candidate answers are sampled at each of five turns, then a naive approach would require $3^5$ sampled trajectories. Finally, practical systems often need to operate under a limited computation budget, requiring the model to abstain or return a ``cannot answer'' decision when sufficient confidence cannot be achieved.

To address these challenges, we propose \textbf{M}ult\textbf{i}-Turn Language Models with \textbf{C}onformal \textbf{P}rediction (MiCP), a multi-level CP framework for multi-turn LLMs reasoning. MiCP assigns distinct error budgets to different turns, enabling the model to stop early when sufficient evidence has been accumulated while still maintaining a rigorous overall coverage guarantee. Intuitively, early turns are encouraged to answer easy questions quickly, whereas later turns reserve additional error budget for more difficult questions that require further retrieval or reasoning. The framework naturally supports both adaptive RAG and ReAct-style agents, and can additionally incorporate a rejection option for questions that remain too uncertain within the allowed budget. We further introduce a new metric that explicitly reflects the trade-off between reliability and efficiency, making it particularly suitable for evaluating multi-turn reasoning systems.

In summary, our work makes the following contributions:
\squishlist
\item We propose MiCP, the first CP framework for multi-turn LLMs. MiCP provides formal coverage guarantees while simultaneously enabling early stopping in iterative pipelines.

\item We develop a multi-level error allocation strategy that distributes the overall conformal error budget across turns. This allows LLMs to decide whether to stop or continue at each turn while preserving the desired end-to-end coverage guarantee, even under a fixed computation budget and optional rejection mechanism.

\item We introduce a new evaluation metric for multi-turn LLMs that jointly measures coverage validity and efficiency. The metric explicitly rewards methods that answer correctly with fewer turns and penalizes unnecessary retrieval or reasoning steps.

\item Extensive experiments on both adaptive RAG and ReAct benchmarks demonstrate that MiCP consistently achieves the target coverage level while substantially reducing the number of turns, inference cost, and prediction set size on both single-hop and multi-hop question answering tasks.
\squishend
\section{Related Work}
\subsection{Multi-turn LLMs}
In multi-turn settings, LLMs can iteratively request additional information before committing to a final answer. A prominent category is adaptive RAG, where models dynamically determine whether to retrieve and how much text to retrieve from external sources based on their own output signals, rather than answering directly in a single pass. IRCoT \citep{DBLP:conf/acl/TrivediBKS23} interleaves retrieval with chain-of-thought reasoning, fetching additional passages when the reasoning steps so far are insufficient to produce the answer. Adaptive-RAG \citep{DBLP:conf/naacl/JeongBCHP24} trains a T5-large classifier to predict whether retrieval is needed and how many turns to perform. FLARE \citep{jiang-etal-2023-active} triggers retrieval whenever token-level generation probability falls below a threshold. DRAGIN \citep{su2024dragindynamicretrievalaugmented} leverages attention weights to identify salient keywords for formulating retrieval queries. Rowen \citep{DBLP:journals/corr/abs-2402-10612-rowen} uses cross-lingual answer consistency to decide whether retrieval is necessary, while SeaKR \citep{DBLP:journals/corr/abs-2406-19215} relies on internal hidden states for the same purpose. QucoRAG \citep{min2025quco} detects hallucination by measuring how frequently generated tokens overlap with training data. \citep{li2026retrievalgenerationunifiedframework} set when to retrieve, what to retrieve, when to stop, into control token without external controller. \citep{moskvoretskii2025adaptive} provides a comprehensive study on how uncertainty estimation affects adaptive RAG performance. Beyond retrieval-augmented settings, ReAct \citep{yao2022react} enables LLMs to interleave reasoning with action steps, generating queries to external tools when additional information is needed, and Reflexion \citep{shinn2023reflexion} allows LLMs to iteratively self-reflect on their own outputs to refine answers across multiple turns. In this work, we primarily evaluate our framework under the adaptive RAG setting and additionally experiment with the ReAct paradigm to demonstrate its generalizability.
\subsection{CP for LLM}
CP has become an important tool for guaranteeing the factuality of predictions from LLMs in high-stakes scenarios~\citep{zhou2025conformal, campos2024conformal}, and has been applied to various LLM tasks~\citep{sheng2025analyzinguncertaintyllmasajudgeinterval, kumar2023conformal, ye2024benchmarking}. In early applications, \citep{quach2023conformal} repeatedly sample outputs until the prediction set is confident enough to contain the ground truth, which is computationally expensive~\citep{su-etal-2024-api}. Several subsequent methods~\citep{su-etal-2024-api, li-etal-2024-traq} instead draw a fixed number of samples and use frequency to measure uncertainty, improving efficiency. Orthogonally, \citep{schuster2022confident} allow the transformer to exit early at the cost of a bounded performance drop.
However, the nearly infinite output space of LLMs makes achieving the target coverage intractable. A common strategy is to replace uncertain answers with more general ones. \citep{zhang2024conformal} propose a conformal structure prediction framework for tasks whose labels lie in a directed acyclic graph. \citep{mohri2024language, rubin2025conformal}, which can be viewed as a special variant of this approach, remove uncertain statements from long-form generations to maintain factuality or coherence. Another variant by \citep{li-etal-2024-traq} assigns a ``Can't Answer'' label to questions that the LLM cannot reliably answer within the fixed sample set. Our method extends the ``Can't Answer'' mechanism to multi-turn LLM interactions, allowing the model to early stop when sufficient confidence is reached while preserving the coverage guarantee in the no-early-stop setting.


\section{Preliminary}
\textbf{Foundations of Split CP.}
CP is a distribution-free framework for constructing prediction sets with finite-sample coverage guarantees \citep{shafer2008tutorial, angelopoulos2021gentle}. It operates as a post-hoc wrapper around any pre-trained model, without requiring any additional pretraining. More specifically, Split CP separates model training from threshold calibration. Given a dataset $\{(x_i, y_i)\}_{i=1}^{n}$, it is partitioned into a training set $\mathcal{D}_{\mathrm{train}} = \{(x_i, y_i)\}_{i=1}^{n_{\mathrm{train}}}$, a calibration set $\mathcal{D}_{\mathrm{cal}} = \{(x_i, y_i)\}_{i=1}^{n_{\mathrm{cal}}}$, and a test set $\mathcal{D}_{\mathrm{test}} = \{(x_i, y_i)\}_{i=1}^{n_{\mathrm{test}}}$, where $x_i$ is the input question, $y_i$ is its corresponding answer, $n_*$ denotes the size of each respective subset, and $n = n_{\mathrm{train}} + n_{\mathrm{cal}} + n_{\mathrm{test}}$.
Let $S(f, (x, y))$ denote a nonconformity score that measures how unusual a label $y$ is for input $x$ under a model $f$ trained on $\mathcal{D}_{\mathrm{train}}$. In a classification setting, for instance, a common choice is
\begin{equation}
\label{eq:score_example}
S(f, (x, y)) = 1 - f_y(x),
\end{equation}
where $f_y(x)$ is the predicted probability that $f$ assigns label $y$ to input $x$. Using the calibration set $\{(x_i, y_i)\}_{i=1}^{n_{\mathrm{cal}}}$, one evaluates the nonconformity scores $\{S(f,(x_i, y_i))\}_{i=1}^{n_{\mathrm{cal}}}$.
A threshold is then obtained by computing a quantile over these calibration scores at a user-specified error rate level $\alpha \in (0,1)$:
\begin{equation}
\label{eq:conformal_quantile}
\hat{q}_{\alpha} = \mathrm{Quantile}\!\left(\{s_i\}_{i=1}^{n_{\mathrm{cal}}};\; \frac{\lceil (n_{\mathrm{cal}} + 1)(1 - \alpha) \rceil}{n_{\mathrm{cal}}}\right).
\end{equation}
For a new test input $x_{n+1}$ and each candidate label $y \in \mathcal{Y}$, the nonconformity score is computed as $s_{n+1}(y) = S(f, (x_{n+1}, y))$. The prediction set is then defined as
\begin{equation}
\label{eq:prediction_set}
\mathcal{C}_{\alpha}(x_{n+1}) = \bigl\{y : s_{n+1}(y) < \hat{q}_{\alpha}\bigr\},
\end{equation}
which satisfies the marginal coverage guarantee
$
\label{eq:coverage_guarantee}
\mathbb{P}\bigl(y_{n+1} \in \mathcal{C}_{\alpha}(x_{n+1})\bigr) \geq 1 - \alpha.$
CP is not only expected to achieve the coverage guarantee but also keep its prediction sets as concise as possible so that the outputs are useful.
\paragraph{CP for LLMs.}
\label{cpllm}
Applying CP to LLMs is challenging because outputs are open-ended text rather than fixed label sets. Recent work \citep{quach2023conformal, su-etal-2024-api} samples multiple responses from the LLM, clusters semantically equivalent answers, and uses the cluster frequency as the basis for nonconformity scores. Given $M$ sampled answers clustered into groups $\{C_1, \ldots, C_K\}$, the nonconformity score for a cluster $C_j$ is defined as:$
S(C_j) = 1 - f(C_j),$
where $f(C_j) = |C_j| / M$ is the fraction of samples assigned to cluster $C_j$. A prediction set is then constructed by including all clusters whose frequency exceeds a conformally calibrated threshold, providing finite-sample coverage guarantees without requiring access to internal model logits. In CP for LLMs, a concentration problem arises because frequency-based nonconformity scores are discrete \citep{su-etal-2024-api}. To address this, \citep{su-etal-2024-api} propose combining frequency with negative entropy (NE) and semantic similarity as tie-breaking mechanisms. 
NE is calculated as:
\begin{equation}
\mathrm{NE}_t(x) = -\frac{1}{\log M} \sum_{j=1}^{K_t} f(C_j) \log f(C_j),
\end{equation}
where $\mathrm{NE}_t(x) \in [0, 1]$, with values near $0$ indicating that samples concentrate on a single answer (high confidence) and values near $1$ indicating a uniform spread across clusters (high uncertainty). The NE penalized frequency used in this work is formulated as:
\begin{equation}
\label{nepenal}
    f(C_i) = f(C_i)-\eta\cdot NE,
\end{equation}
where $\eta$ is a hyperparameter.

\section{Method}
We propose MiCP, a conformal framework for multi-turn LLMs that consists of three components: turn-level confidence estimation, adaptive early stopping, and final prediction set construction. At each turn, MiCP first evaluates the reliability of the current prediction set using a calibrated confidence criterion. It then decides whether to stop or continue by comparing the turn-specific score against conformal thresholds with different error budgets across iterations. If additional information or reasoning is needed, the model performs another retrieval, action, or reasoning step, as in adaptive RAG or ReAct. Once the model stops, MiCP constructs a final prediction set over the sampled answers, with an optional ``Can't Answer'' label when no reliable prediction can be made within the allotted budget.
\begin{figure}
    \centering
    \includegraphics[width=1\linewidth]{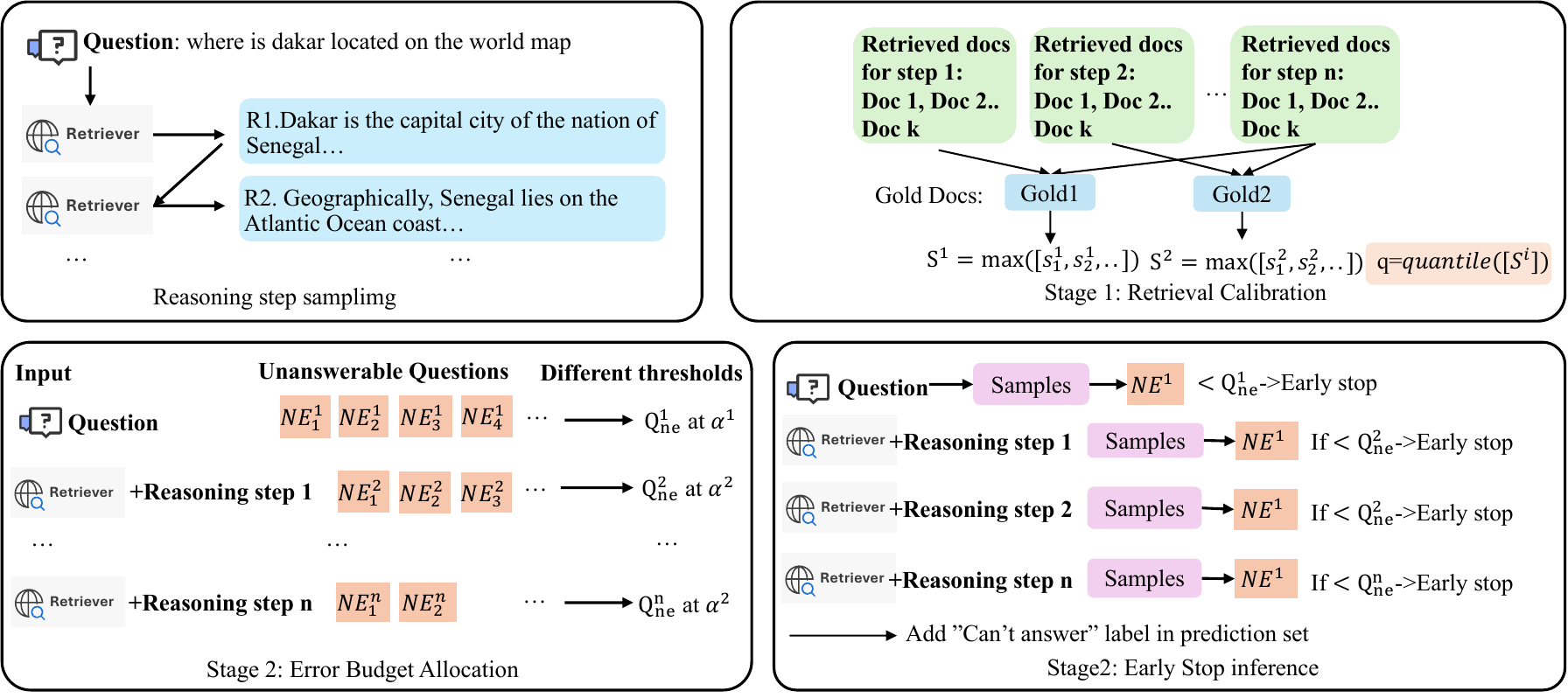}
    \caption{MiCP pipeline.}
    \label{fig:pipeline}
\end{figure}

\subsection{Problem Setup and Notation}

We consider a single reasoning trajectory per example. For question $x_i$, the model iteratively generates a sequence of reasoning steps $\mathbf{r}^{i} = (r_0^i, r_1^i, \ldots, r_{T-1}^i)$, where $r_0^i$ is the original question and each subsequent state is conditioned on prior reasoning and any retrieved passages or external actions. For example, in adaptive RAG, $r_t^i$ corresponds to a retrieval query, and in ReAct, $r_t^i$ may include reasoning and tool use.

MiCP calibration proceeds in two stages: turn-level retrieval filtering, followed by prediction set construction with adaptive early stopping.

\paragraph{Stage 1: Retrieval Filtering.}
Let $\{(x_i, \mathcal{P}_i^*)\}_{i=1}^{n_{\mathrm{cal}}}$ be the calibration set, where
\begin{equation}
    \mathcal{P}_i^* = \{p \in \mathcal{G}_i \mid \exists\, t,\; p \in \mathrm{Top}\text{-}K(r_t^i)\}
\end{equation}
is the subset of gold passages $\mathcal{G}_i$ retrievable from at least one reasoning step. We calibrate a threshold $\hat{q}^{\mathrm{ret}}$ and retain only passages with relevance score $s_t(p) \geq \hat{q}^{\mathrm{ret}}$, where $s_t(p)$ denotes the relevance score between passage $p$ and the query at turn $t$, ensuring approximately $1-\alpha_{\mathrm{ret}}$ of retrievable gold passages are kept while filtering low-relevance evidence.

\paragraph{Stage 2: Prediction Set and Early Stopping.}
Let $\{(x_i, y_i, \{\mathcal{A}_t^i\}_{t=0}^{T})\}_{i=1}^{n_{\mathrm{cal}}}$ be the calibration set, where $y_i$ is the gold answer and $\mathcal{A}_t^i$ is the set of $M$ sampled answers at turn $t$ using filtered context from Stage~1. We jointly learn: (1) turn-specific stopping thresholds $\{\hat{q}_t\}_{t=0}^{T}$ that determine whether the model should stop after turn $t$, and (2) a final prediction set
\begin{equation}
    \mathcal{C}(x_i) \subseteq \bigcup_{t=0}^{T}\mathcal{A}_t^i \cup \{\texttt{Can't Answer}\}.
\end{equation}
At each turn $t$, MiCP computes a calibrated confidence score from the sampled answers and stops whenever the score exceeds $\hat{q}_t$. If the model stops at turn $t < T$, the prediction set is constructed from $\bigcup_{\tau=0}^{t}\mathcal{A}_\tau^i$; after the final turn $T$, it may additionally include ``Can't Answer''. The prediction set is calibrated to satisfy
\begin{equation}
\mathbb{P}\!\left(
y_i \in \mathcal{C}(x_i) \;\text{if}\; y_i \in \bigcup_{t=0}^{T}\mathcal{A}_t^i,\;
\texttt{Can't Answer} \in \mathcal{C}(x_i) \;\text{if}\; y_i \notin \bigcup_{t=0}^{T}\mathcal{A}_t^i
\right) \geq 1-\alpha.
\end{equation}
That is, MiCP guarantees the prediction set either contains the correct answer when it appears among the sampled answers, or abstains via ``Can't Answer'' when no sampled answer is correct.

\subsection{Retrieval Threshold Calibration}
\label{sec:method:phase1}

A key challenge in multi-turn pipelines is that each retrieval turn returns a fixed set of top-$K$ passages, many of which may be irrelevant. As retrieval proceeds, these irrelevant passages accumulate and pollute the LLM's context, degrading both answer quality and stopping decisions. Prior work has applied CP to filter irrelevant passages in single-round retrieval~\citep{chakraborty2025principled}; we extend this to the multi-turn setting, applying retrieval filtering at every turn while preserving coverage over the entire reasoning process.

Because a gold passage $p$ may be retrieved at multiple turns with different relevance scores, we define its optimistic relevance score as the maximum across all turns:
\begin{equation}
s^*(p, x_i) = \max_{t:\, p \in \mathrm{Top}\text{-}K(r_t^i)} s_t(p).
\end{equation}
We aggregate these scores across the calibration set:
$
\mathcal{S}_{\mathrm{ret}} =
\{s^*(p, x_i) \mid p \in \mathcal{P}_i^*,\; i=1,\ldots,n_{\mathrm{cal}}\},
$
and define the conformal retrieval threshold as
\begin{equation}
\hat{q}_{\mathrm{ret}}
=
\mathrm{Quantile}
\left(
\mathcal{S}_{\mathrm{ret}}
\;\middle|\;
\frac{(n-1)\alpha_{\mathrm{ret}}}{n}
\right).
\label{eq:q_hat_ret}
\end{equation}

At test time, for turn $t$, MiCP retains only passages whose relevance scores exceed the threshold:
\begin{equation}
\mathcal{C}^{\mathrm{ret}}_t(x)
=
\left\{
p \in \mathrm{Top}\text{-}K(r_t)
\;\middle|\;
s_t(p) > \hat{q}_{\mathrm{ret}}
\right\}.
\end{equation}

This guarantees that at least a $1-\alpha_{\mathrm{ret}}$ fraction of retrievable gold passages are retained, while adaptively removing low-relevance passages that would otherwise hinder subsequent reasoning.

\subsection{Sequential NE Calibration with Error Budget Allocation}
\label{sec:method:phase2}

The multi-turn retrieval loop introduces a sequential decision structure: at each turn $t$, the system must decide whether the question can be answered with sufficient confidence or requires further retrieval. We use the normalized entropy (NE) of the LLM's sampled answer distribution as the uncertainty signal.

A naive approach uses one global error rate for all turns, but this ignores differences in question difficulty. Some questions can be answered after one retrieval step, while others require multiple hops. Using the same budget everywhere wastes it on easy, early cases and leaves less for harder, later ones.

To address this, we decompose the total error budget $\alpha$ across turns, allowing each turn to operate with its own calibrated NE threshold. We allocate a per-turn budget $\alpha_t$ for $t = 0, \ldots, T$, subject to:
\begin{equation}
\label{eq:error_decomposition}
\sum_{t=0}^{T} (1 - c^t_{\mathrm{ans}}) \cdot \alpha_t \;\leq\; (1 - c_{\mathrm{ans}}^{\mathrm{final}}) \cdot \alpha,
\end{equation}
where $c^t_{\mathrm{ans}}$ is the fraction of questions not early-stopped that are answerable at turn $t$, and $c_{\mathrm{ans}}^{\mathrm{final}}$ is the fraction of all calibration questions for which the gold answer appears in any turn's samples. The term $(1 - c^t_{\mathrm{ans}}) \cdot \alpha_t$ accounts for unanswerable questions early-stopped at turn $t$, whose prediction sets cannot include the gold answer regardless of the threshold.

Let $\mathcal{U}_t$ be the subset of questions still active at turn $t$ for which the gold answer has not appeared in any sampled answers up to and including $t$. The per-turn NE threshold is:
\begin{equation}
\label{NE}
\hat{q}_t = \mathrm{Quantile}\!\Bigl(\{\mathrm{NE}_t(x_i) \mid x_i \in \mathcal{U}_t\};\; \frac{(n_{\mathrm{cal}}-1)\alpha_t}{n_{\mathrm{cal}}}\Bigr).
\end{equation}

\subsection{Composite Evaluation Metric}
\label{sec:method:metric}

We introduce a composite metric to evaluate adaptive stopping in multi-turn LLMs. Some basic metrics provide misleading optimization incentives. The average number of turns encourages the model to stop as early as possible. Answer rate ignores that the proportion of answerable and unanswerable questions may vary across turns.

To balance these objectives, we propose:
\begin{equation}
\mathcal{L} = \gamma \cdot \bar{T} \;-\;  \frac{\sum_{t=0}^T n_{\mathrm{corr}}^{t} \,/\, c_{\mathrm{ans}}^{t}}{\sum_{t=0}^T n_{\mathrm{wrong}}^{t} \,/\, (1 - c_{\mathrm{ans}}^{t})}
\label{eq:composite_objective}
\end{equation}
where $\bar{T}$ is the average number of turns used, $n_{\mathrm{corr}}^t$ and $n_{\mathrm{wrong}}^t$ are the numbers of correctly and incorrectly answered examples that stop at turn $t$, and $\gamma$ controls the efficiency--quality trade-off. The second term measures how well the stopping policy distinguishes answerable from unanswerable questions at each turn, rewarding correct answers at turns where answering is feasible while penalizing incorrect answers at turns where most questions remain unresolved.
Although introduced in the context of MiCP, this metric is general and can evaluate stopping policies in adaptive RAG, ReAct, and other multi-turn LLM systems.

\paragraph{Grid Search Optimization}
\label{sec:method:grid}
We perform a grid search over $\alpha_0, \ldots, \alpha_{T-1}$, deriving $\alpha_{T}$ from Eq.~\eqref{eq:error_decomposition}, and select the allocation minimizing $\mathcal{L}$ on a held-out optimization set $\mathcal{D}_{\mathrm{opt}}$.

\subsection{Final Answer Prediction Set Confidence Threshold Calibration}
\label{sec:method:freq_cal}
For each calibration sample whose gold answer appears in at least one sampled answer before stopping, let $t_i^*$ denote the stopping turn (or $T$ if no early stopping occurs). At each turn $t \leq t_i^*$, we cluster the sampled answers and compute a confidence score for each cluster based on its frequency, optionally adjusted by the NE. A high-confidence cluster corresponds to many consistent samples and indicates that the model may already have sufficient evidence.

The confidence of the gold-answer cluster is:
\begin{equation}
f(C_{\mathrm{gold}}^i)
=
\max_{0 \leq t \leq t_i^*}
\max_{C_j^t:\, y_i \in C_j^t}
f(C_j^t),
\label{eq:freq_score}
\end{equation}
where the outer maximum ranges over turns up to the stopping point, and the inner maximum ranges over all answer clusters $C_j^t$ containing the gold answer $y_i$.

The conformal threshold is:
\begin{equation}
\hat{q}_{\mathrm{freq}}
=
\mathrm{Quantile}
\left(
\{f(C_{\mathrm{gold}}^i)\}_{i \in \mathcal{D}_{\mathrm{cal}}^{\mathrm{ans}}},\;
\frac{(n_{\mathrm{cal}}-1)\alpha}{n_{\mathrm{cal}}}
\right),
\label{eq:q_hat_freq}
\end{equation}
where $\mathcal{D}_{\mathrm{cal}}^{\mathrm{ans}}$ denotes calibration samples whose gold answer appears in the sampled answers before stopping.

\subsection{Prediction Set Construction}
\label{sec:method:predset}
At test time, for a test sample $x$, MiCP stops at the first turn whose highest-confidence cluster exceeds the stopping threshold:
$
t^* = \min \{ t : f(C_{\max}^t) \geq \hat{q}_t \}.
$
If no threshold is satisfied, the model continues through all $T$ turns.

For every turn $t \leq t^*$, sampled answers are clustered using the procedure in Section~\ref{cpllm}. The final prediction set is:
\begin{equation}
\mathcal{C}(x)=
\begin{cases}
\displaystyle
\bigcup_{t=0}^{t^*}
\left\{
C_j^t \mid f(C_j^t) \geq \hat{q}_{\mathrm{freq}}
\right\},
& \text{if } t^* < T, \\[1.2em]
\displaystyle
\bigcup_{t=0}^{T}
\left\{
C_j^t \mid f(C_j^t) \geq \hat{q}_{\mathrm{freq}}
\right\}
\cup
\{\texttt{Can't Answer}\},
& \text{otherwise}.
\end{cases}
\label{psc}
\end{equation}

If the model reaches sufficient confidence before the final turn, MiCP returns only the high-confidence answer clusters accumulated up to that point. Otherwise, after exhausting all turns, it additionally includes ``Can't Answer'' to preserve the coverage guarantee.
\section{Experiment} 

We investigate the following research questions:
\squishlist
\item \textbf{RQ1:} Can our proposed method achieve the desired coverage guarantee for both retrieval and prediction sets in multi-turn reasoning settings?
\item \textbf{RQ2:} How does the efficiency of our method, in terms of both prediction set size and number of inference steps, compare with existing baselines?
\item \textbf{RQ3:} To what extent does grid search improve the optimization of our proposed objective?
\squishend

\subsection{Experimental Setup}
\label{sec:experiments:setup}
We evaluate MiCP on five question answering benchmarks spanning both single-hop and multi-hop reasoning: Natural Questions (NQ)~\citep{kwiatkowski2019natural}, TriviaQA~\citep{joshi2017triviaqa}, HotpotQA~\citep{yang2018hotpotqa}, MuSiQue~\citep{trivedi2022musique}, and 2WikiMultiHopQA (2WikiMHQA)~\citep{ho2020constructing}. To demonstrate the generality of MiCP beyond adaptive RAG, we additionally evaluate on the multi-turn reasoning framework ReAct~\citep{yao2022react}, where the model alternates between reasoning and information-seeking actions. For each dataset, we randomly sample 900 examples and split them into an optimization set ($|\mathcal{D}_{\mathrm{opt}}| = 300$), a calibration set ($|\mathcal{D}_{\mathrm{cal}}| = 300$), and a test set ($|\mathcal{D}_{\mathrm{test}}| = 300$). We use three popular LLMs of comparable scale: Gemma-2-9B-IT~\citep{team2024gemma}, Qwen3.5-9B~\citep{qwen3.5}, and GPT-4o-mini \citep{achiam2023gpt}, with answer clustering performed by the all-MiniLM-L6-v2 sentence transformer using agglomerative clustering at a cosine similarity threshold of $0.9$.
Other settings can be found in Appendix \ref{setting}.

\begin{figure}[t]
\centering
\begin{subfigure}[b]{0.32\textwidth}
    \centering
    \includegraphics[width=\textwidth]{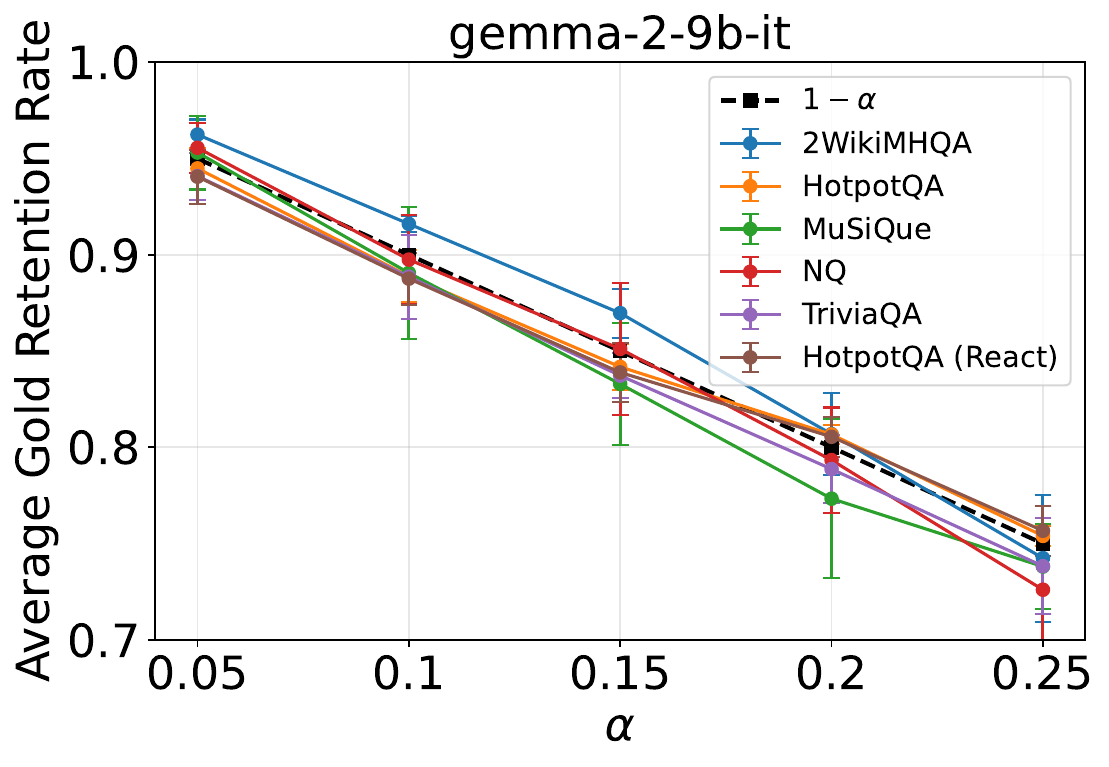}
    \caption{Gemma-2-9B-IT}
    \label{fig:coverage_gemma}
\end{subfigure}
\hfill
\begin{subfigure}[b]{0.32\textwidth}
    \centering
    \includegraphics[width=\textwidth]{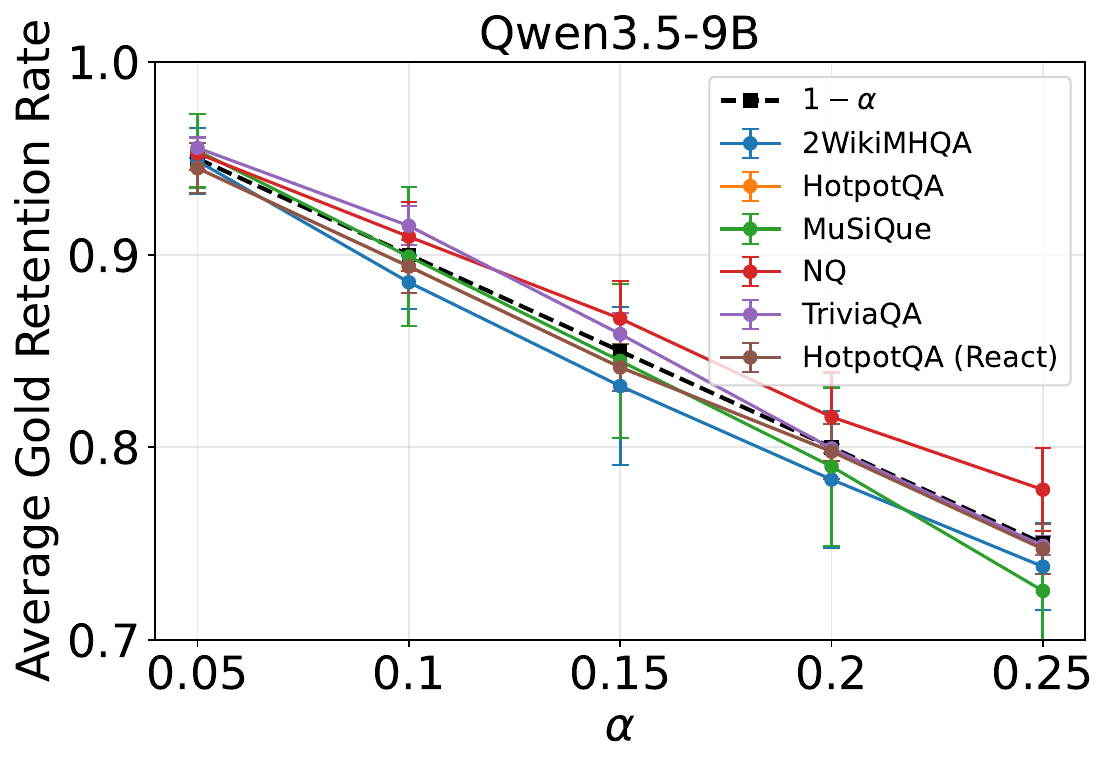}
    \caption{Qwen3.5-9B}
    \label{fig:coverage_qwen}
\end{subfigure}
\hfill
\begin{subfigure}[b]{0.32\textwidth}
    \centering
    \includegraphics[width=\textwidth]{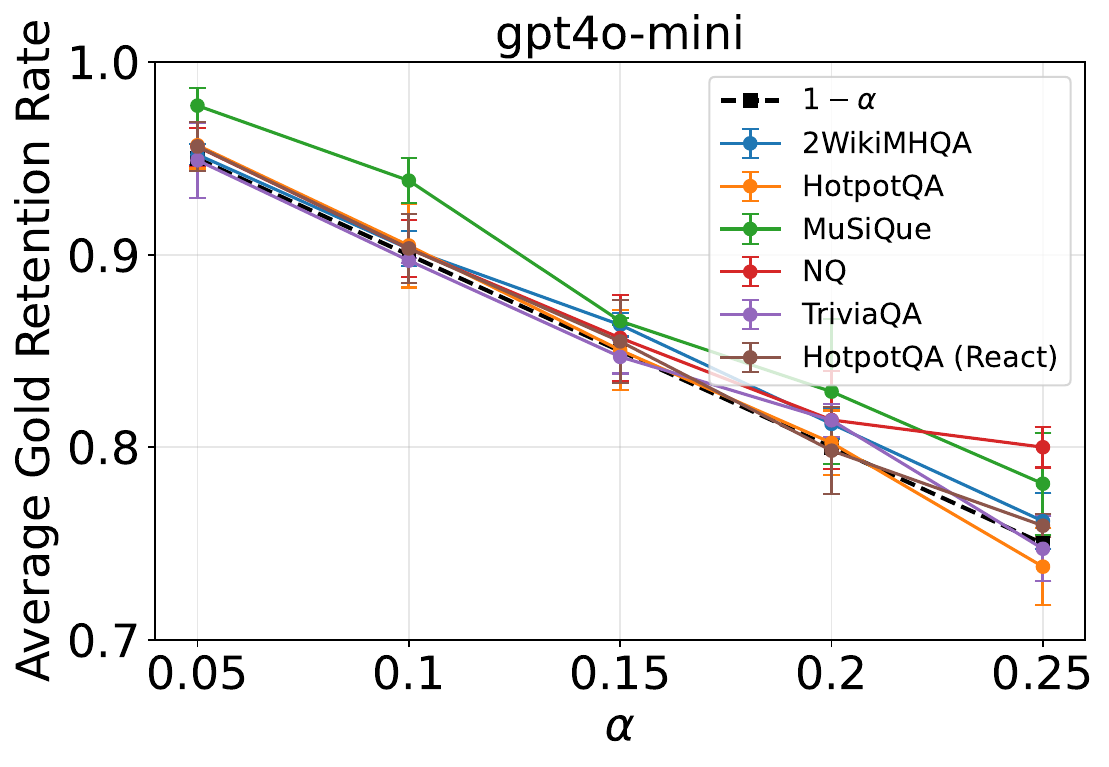}
    \caption{GPT-4o-mini}
    \label{fig:coverage_gpt}
\end{subfigure}
\caption{Empirical gold retention rate vs.\ the target $1-\alpha$ across five datasets on adaptive RAG and ReAct for retrieval calibration. All models maintain coverage at or above the $1-\alpha$ guarantee for all tested error rates.}
\label{fig:coverage}
\end{figure}
\begin{figure}[t]
\centering
\begin{subfigure}[b]{0.32\textwidth}
    \centering
    \includegraphics[width=\textwidth]{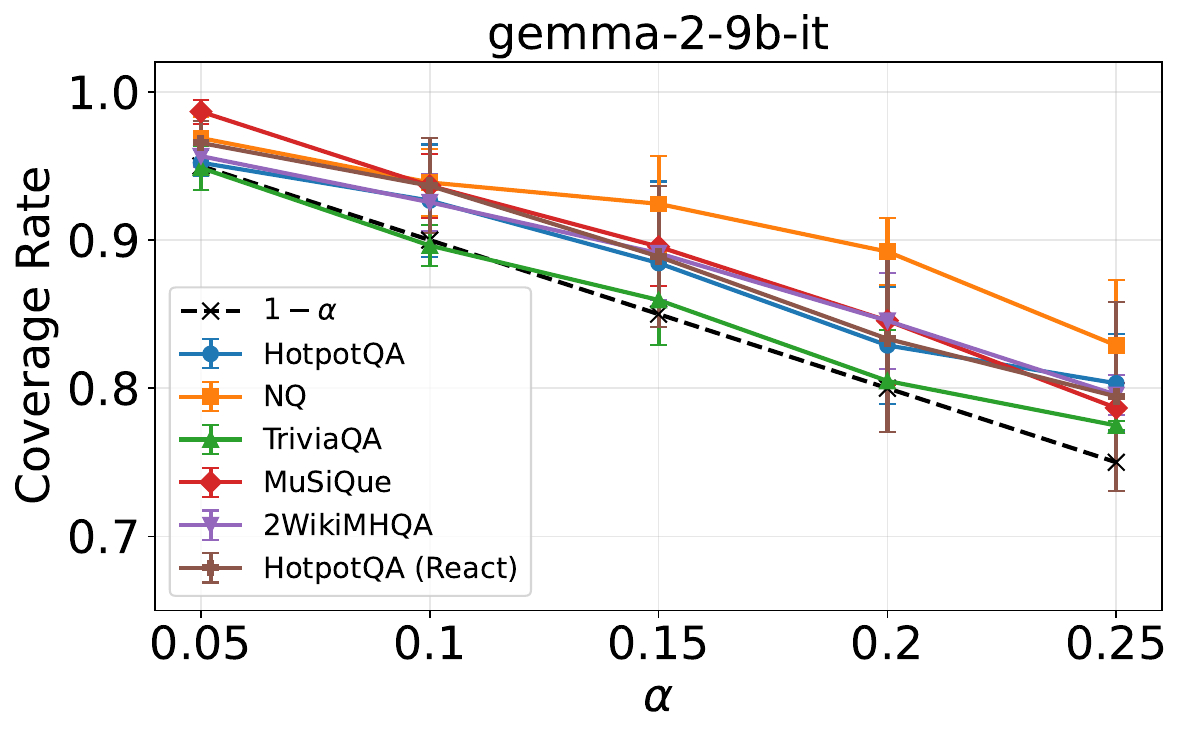}
    \caption{Gemma-2-9B-IT}
    \label{fig:coverage_gemma_pe}
\end{subfigure}
\hfill
\begin{subfigure}[b]{0.32\textwidth}
    \centering
    \includegraphics[width=\textwidth]{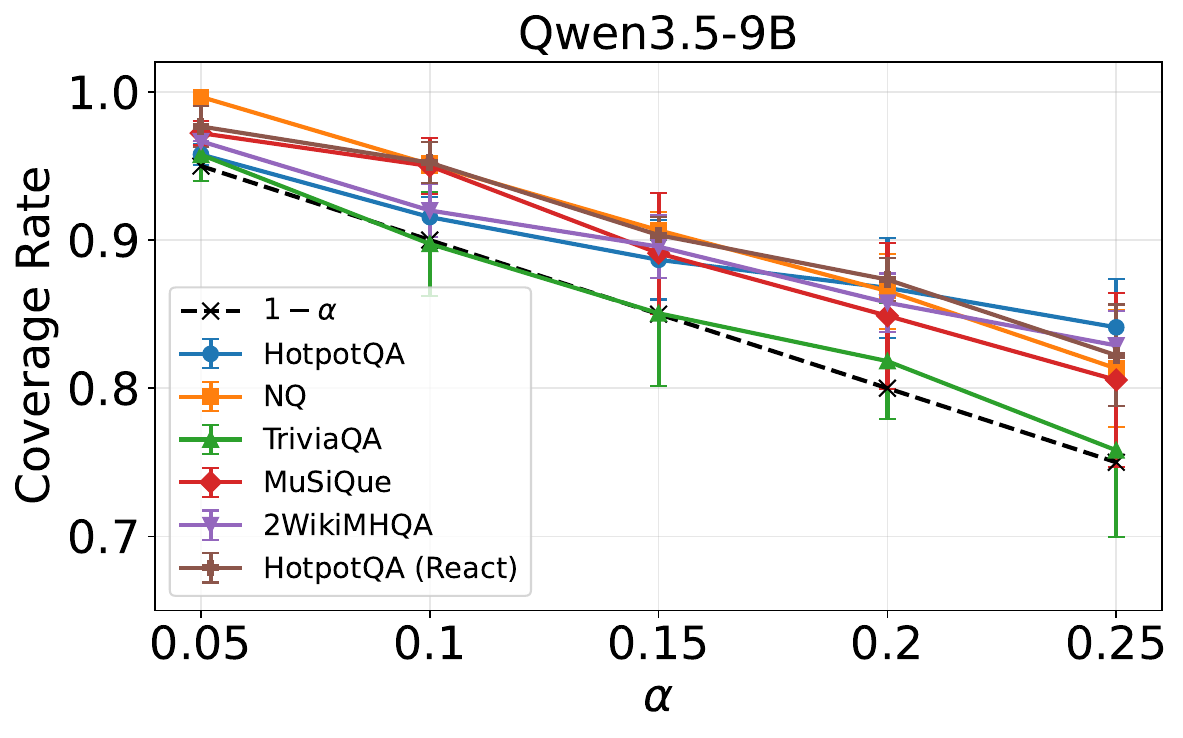}
    \caption{Qwen3.5-9B}
    \label{fig:coverage_qwen_pe}
\end{subfigure}
\hfill
\begin{subfigure}[b]{0.32\textwidth}
    \centering
    \includegraphics[width=\textwidth]{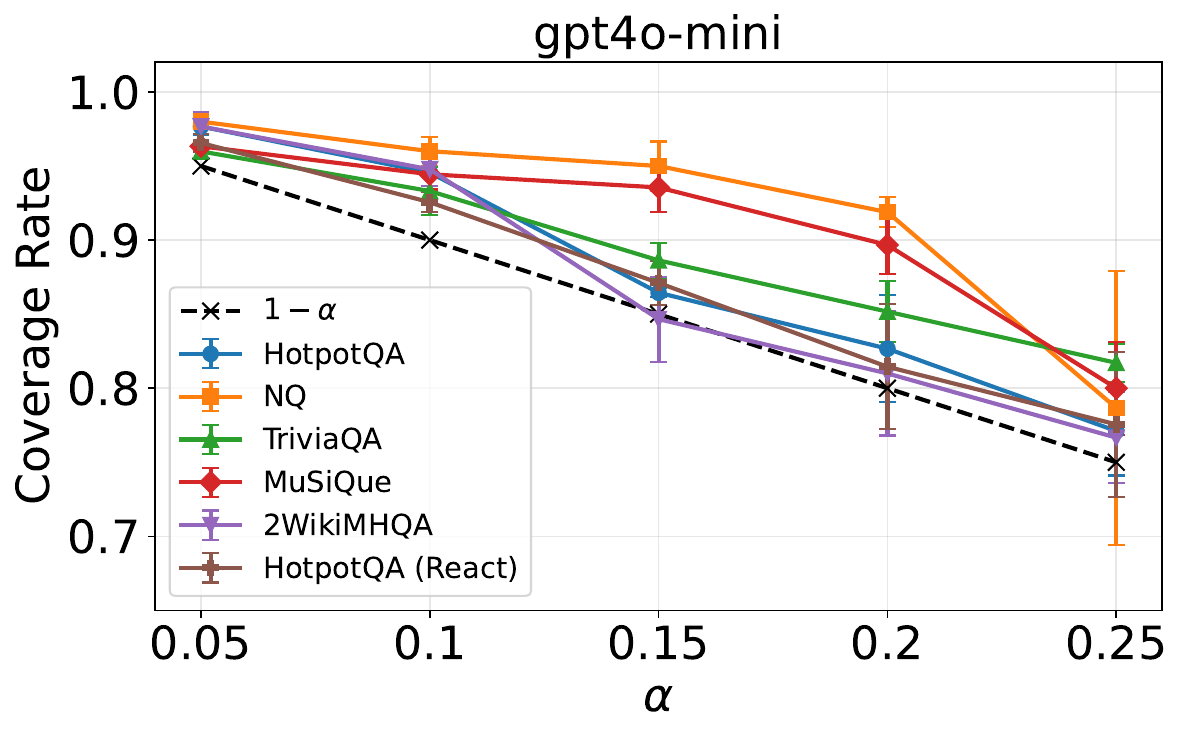}
    \caption{GPT-4o-mini}
    \label{fig:coverage_gpt_pe}
\end{subfigure}
\caption{Empirical coverage rate vs.\ the target $1-\alpha$ across five datasets on adaptive RAG and ReAct for the final prediction set. All models maintain coverage at or above the $1-\alpha$ guarantee for all tested error rates.}
\label{fig:coverage_pe}
\end{figure}

\subsection{RQ1}
\label{sec:exp:rq1}

Figure \ref{fig:coverage} and Figure \ref{fig:coverage_pe} present the empirical coverage results for retrieval filtering and prediction set construction, respectively. For retrieval, the average gold retention rate closely tracks the $1-\alpha$ target across all five datasets and all models, confirming that the conformal retrieval threshold $\hat{q}_{\mathrm{ret}}$ successfully preserves retrievable gold passages while filtering irrelevant ones. For the prediction set, the coverage rate similarly stays at or above $1-\alpha$ for nearly all $\alpha$ values and datasets. It shows a little overcoverage problem, which might be caused by the concentration problem coming from the frequency \citep{su-etal-2024-api}. These results confirm that MiCP maintains the desired $1-\alpha$ coverage guarantee across both the retrieval and prediction stages, even with early stopping enabled, which comes from our multi-hop retrieval calibration and error budget allocation.

\subsection{RQ2}
\label{sec:exp:rq2}

\begin{table*}[t]
\centering
\caption{Average number of retrieval turns (Avg \#Turns), prediction set size, and composite objective across five datasets, averaged over all $\alpha$ values. \textbf{Bold} indicates the best result per row. Lower is better for all metrics. w/o ES means no early stop.}
\label{tab:main_results}
\resizebox{\textwidth}{!}{
\begin{tabular}{ll|cccc|cccc|cccc}
\toprule
& & \multicolumn{4}{c|}{\textbf{Qwen3.5-9B}} & \multicolumn{4}{c|}{\textbf{Gemma-2-9B-IT}} & \multicolumn{4}{c}{\textbf{GPT-4o-mini}} \\
& \textbf{Dataset} & Ours & w/o ES & Avg Iter & Answer Rate & Ours & w/o ES & Avg Iter & Answer Rate & Ours & w/o ES & Avg Iter & Answer Rate \\
\midrule
\multirow{6}{*}{\rotatebox{90}{\textbf{Avg \#Turns}}}
& HotpotQA         & 2.562 & 3.000 & \textbf{2.331} & 2.433 & 2.782 & 3.000 & \textbf{2.481} & 2.854 & 2.497 & 3.000 & \textbf{2.173} & 2.607 \\
& HotpotQA (React) & 2.522 & 3.000 & \textbf{2.313} & 2.557 & 2.752 & 3.000 & \textbf{2.472} & 2.757 & 2.587 & 3.000 & \textbf{2.167} & 2.877 \\
& NQ               & 2.621 & 3.000 & \textbf{2.459} & 2.770 & 2.840 & 3.000 & \textbf{2.597} & 2.856 & 3.000 & 3.000 & \textbf{2.849} & 2.924 \\
& TriviaQA         & 2.528 & 3.000 & \textbf{2.356} & 2.587 & 2.898 & 3.000 & \textbf{2.687} & 2.806 & \textbf{2.494} & 3.000 & \textbf{2.494} & 2.641 \\
& MuSiQue          & 2.705 & 3.000 & \textbf{2.519} & 2.861 & 2.795 & 3.000 & \textbf{2.488} & 2.681 & 2.870 & 3.000 & \textbf{2.278} & 2.936 \\
& 2WikiMHQA        & 1.778 & 3.000 & \textbf{1.472} & 1.477 & 2.553 & 3.000 & \textbf{2.275} & 2.790 & \textbf{2.293} & 3.000 & \textbf{2.293} & 2.767 \\
\midrule
\multirow{6}{*}{\rotatebox{90}{\textbf{Avg Set Size}}}
& HotpotQA         & 10.75 & 11.45 & \textbf{10.56} & 10.57 & 7.99 & 8.10 & \textbf{7.68} & 7.88 & 3.503 & 2.738 & \textbf{2.469} & 4.382 \\
& HotpotQA (ReAct) & 11.94 & 12.85 & \textbf{11.50} & 11.79 & 6.82 & 6.85 & \textbf{6.59} & 6.69 & 4.426 & 3.781 & \textbf{2.576} & 4.776 \\
& NQ               & 19.63 & 21.16 & \textbf{19.02} & 19.82 & 9.38 & 9.66 & \textbf{8.77} & 9.27 & 5.753 & 5.753 & \textbf{5.683} & 5.740 \\
& TriviaQA         & 5.64  & 5.77  & \textbf{5.47}  & 5.50  & 3.26 & 3.22 & \textbf{3.13} & 3.14 & 1.967 & 1.982 & 2.305 & \textbf{1.942} \\
& MuSiQue          & 43.25 & 42.02 & \textbf{40.09} & 42.44 & 24.87 & 25.46 & \textbf{22.50} & 23.34 & 19.568 & 25.121 & \textbf{10.207} & 25.066 \\
& 2WikiMHQA        & 7.32  & 8.62  & \textbf{6.54}  & 6.57  & 9.09 & 8.89 & \textbf{8.73} & 8.70 & \textbf{4.033} & 4.164 & \textbf{4.033} & 4.078 \\
\midrule
\multirow{6}{*}{\rotatebox{90}{\textbf{Ours Obj.}}}
& HotpotQA         & \textbf{-5.95} & -1.96 & -5.17 & -4.84 & -1.95 & -1.34 & \textbf{-2.47} & -1.78 & \textbf{-2.097} & -1.843 & -2.083 & -2.045 \\
& HotpotQA (ReAct) & -6.95 & -2.68 & \textbf{-7.28} & -4.42 & -2.58 & -1.74 & \textbf{-4.07} & -2.16 & \textbf{-2.932} & -1.518 & -1.993 & -1.902 \\
& NQ               & \textbf{-8.67} & -3.23 & -5.20 & -4.30 & -3.20 & -1.93 & \textbf{-3.27} & -2.50 & \textbf{-0.368} & \textbf{-0.368} & -0.364 & -0.361 \\
& TriviaQA         & \textbf{-3.61} & -2.13 & -3.31 & -2.84 & -2.06 & -0.97 & \textbf{-2.26} & -1.37 & -2.205 & \textbf{-3.467} & -1.617 & -2.356 \\
& MuSiQue          & \textbf{-5.17} & -2.08 & -4.37 & -2.99 & \textbf{-2.43} & -2.41 & -2.29 & -2.37 & \textbf{-2.785} & -1.953 & -1.787 & -2.701 \\
& 2WikiMHQA        & \textbf{-11.79} & -3.01 & -8.16 & -8.19 & -2.51 & -1.58 & \textbf{-3.21} & -1.73 & \textbf{-1.920} & -1.334 & \textbf{-1.920} & -1.449 \\
\bottomrule
\end{tabular}
}
\end{table*}

Table~\ref{tab:main_results} compares three variations of MiCP---Ours (optimizing the composite objective in Eq.~\eqref{eq:composite_objective}), Avg Iter (optimizing for average turns), and Answer Rate (optimizing for the ratio of questions without the ``Can't Answer'' label)---against the No Early Stop baseline where all questions run for $T$ turns. We highlight two findings.

First, MiCP successfully improves the inference efficiency (avg \#turns), with better and more recent models (e.g., qwen3.5-9B, gpt-4o-mini), our method is able to get better inference efficiency. Especially on 2wikiMHQA, it is easier to distinguish answerable and unanswerable questions as it is a fully template-generated multi-hop QA dataset. MiCP also consistently improves the prediction efficiency (prediction set size). It achieves this by filtering answers from unnecessary late turns, which may not include the correct answer and may even bring noisy, wrong answers. 

Second, we find that using average turns as the optimization objective consistently yields the fewest turns and smallest prediction set sizes. This is because optimizing for fewer turns encourages the system to allocate more error budget to early turns, causing the LLM to stop earlier and thereby avoiding the accumulation of noisy, low-confidence answers from unnecessary later turns.




\subsection{RQ3}
\label{sec:exp:rq3}

As shown in Table \ref{tab:main_results}, grid search over Qwen3.5-9B and gpt-4-mini successfully optimizes our proposed objective, yielding more rational early-stopping decisions. Moreover, all variants of MiCP consistently improve upon the baseline objective, demonstrating the necessity of early stopping in multi-turn question answering. On Gemma-2-9B-IT, the improvements are smaller but still consistent, confirming that the grid search generalizes across models with different uncertainty characteristics. Noticing that for the Gemma model, using avg-turn even achieves the best performance on the objective without optimizing it. An explanation can be that it is hard for Gemma-2-9B to distinguish the difference between answerable and unanswerable questions, which makes the optimization not stable.
\section{Conclusion}
We presented MiCP, the first conformal framework for multi-turn LLMs, which jointly calibrates retrieval filtering, adaptive early stopping, and final prediction set construction while maintaining an overall $1-\alpha$ coverage guarantee. MiCP applies broadly to multi-turn pipelines such as adaptive RAG and ReAct, where the model may retrieve, reason, or act over multiple turns before answering. By allocating the error budget across turns and optimizing turn-specific thresholds through grid search, MiCP preserves the same coverage guarantee as a strategy without early stopping while substantially improving efficiency. Experiments on five single-hop and multi-hop QA benchmarks with three popular LLMs show that MiCP achieves the target coverage at both the retrieval and answer stages, while significantly reducing the number of turns, retrieval cost, and maintaining competitive prediction set sizes.
\bibliography{colm2026_conference}

@inproceedings{DBLP:conf/acl/TrivediBKS23,
  author       = {Harsh Trivedi and
                  Niranjan Balasubramanian and
                  Tushar Khot and
                  Ashish Sabharwal},
  editor       = {Anna Rogers and
                  Jordan L. Boyd{-}Graber and
                  Naoaki Okazaki},
  title        = {Interleaving Retrieval with Chain-of-Thought Reasoning for Knowledge-Intensive
                  Multi-Step Questions},
  booktitle    = {Proceedings of the 61st Annual Meeting of the Association for Computational
                  Linguistics (Volume 1: Long Papers), {ACL} 2023, Toronto, Canada,
                  July 9-14, 2023},
  pages        = {10014--10037},
  publisher    = {Association for Computational Linguistics},
  year         = {2023},
  url          = {https://doi.org/10.18653/v1/2023.acl-long.557},
  doi          = {10.18653/V1/2023.ACL-LONG.557},
  bibsource    = {dblp computer science bibliography, https://dblp.org}
}

@inproceedings{DBLP:conf/naacl/JeongBCHP24,
  author       = {Soyeong Jeong and
                  Jinheon Baek and
                  Sukmin Cho and
                  Sung Ju Hwang and
                  Jong Park},
  editor       = {Kevin Duh and
                  Helena G{\'{o}}mez{-}Adorno and
                  Steven Bethard},
  title        = {Adaptive-RAG: Learning to Adapt Retrieval-Augmented Large Language
                  Models through Question Complexity},
  booktitle    = {Proceedings of the 2024 Conference of the North American Chapter of
                  the Association for Computational Linguistics: Human Language Technologies
                  (Volume 1: Long Papers), {NAACL} 2024, Mexico City, Mexico, June 16-21,
                  2024},
  pages        = {7036--7050},
  publisher    = {Association for Computational Linguistics},
  year         = {2024},
  url          = {https://doi.org/10.18653/v1/2024.naacl-long.389},
  doi          = {10.18653/V1/2024.NAACL-LONG.389},
  bibsource    = {dblp computer science bibliography, https://dblp.org}
}

@inproceedings{jiang-etal-2023-active,
    title = "Active Retrieval Augmented Generation",
    author = "Jiang, Zhengbao  and
      Xu, Frank  and
      Gao, Luyu  and
      Sun, Zhiqing  and
      Liu, Qian  and
      Dwivedi-Yu, Jane  and
      Yang, Yiming  and
      Callan, Jamie  and
      Neubig, Graham",
    editor = "Bouamor, Houda  and
      Pino, Juan  and
      Bali, Kalika",
    booktitle = "Proceedings of the 2023 Conference on Empirical Methods in Natural Language Processing",
    month = dec,
    year = "2023",
    address = "Singapore",
    publisher = "Association for Computational Linguistics",
    url = "https://aclanthology.org/2023.emnlp-main.495",
    doi = "10.18653/v1/2023.emnlp-main.495",
    pages = "7969--7992",
}

@misc{su2024dragindynamicretrievalaugmented,
      title={DRAGIN: Dynamic Retrieval Augmented Generation based on the Information Needs of Large Language Models}, 
      author={Weihang Su and Yichen Tang and Qingyao Ai and Zhijing Wu and Yiqun Liu},
      year={2024},
      eprint={2403.10081},
      archivePrefix={arXiv},
      primaryClass={cs.CL},
      url={https://arxiv.org/abs/2403.10081}, 
}

@article{DBLP:journals/corr/abs-2402-10612-rowen,
  author       = {Hanxing Ding and
                  Liang Pang and
                  Zihao Wei and
                  Huawei Shen and
                  Xueqi Cheng},
  title        = {Retrieve Only When It Needs: Adaptive Retrieval Augmentation for Hallucination
                  Mitigation in Large Language Models},
  journal      = {CoRR},
  volume       = {abs/2402.10612},
  year         = {2024},
  url          = {https://doi.org/10.48550/arXiv.2402.10612},
  doi          = {10.48550/ARXIV.2402.10612},
  eprinttype    = {arXiv},
  eprint       = {2402.10612},
  bibsource    = {dblp computer science bibliography, https://dblp.org}
}

@article{DBLP:journals/corr/abs-2406-19215,
  author       = {Zijun Yao and
                  Weijian Qi and
                  Liangming Pan and
                  Shulin Cao and
                  Linmei Hu and
                  Weichuan Liu and
                  Lei Hou and
                  Juanzi Li},
  title        = {SeaKR: Self-aware Knowledge Retrieval for Adaptive Retrieval Augmented
                  Generation},
  journal      = {CoRR},
  volume       = {abs/2406.19215},
  year         = {2024},
  url          = {https://doi.org/10.48550/arXiv.2406.19215},
  doi          = {10.48550/ARXIV.2406.19215},
  eprinttype    = {arXiv},
  eprint       = {2406.19215},
  bibsource    = {dblp computer science bibliography, https://dblp.org}
}

@article{min2025quco,
  title={QuCo-RAG: Quantifying Uncertainty from the Pre-training Corpus for Dynamic Retrieval-Augmented Generation},
  author={Min, Dehai and Zhang, Kailin and Wu, Tongtong and Cheng, Lu},
  journal={arXiv preprint arXiv:2512.19134},
  year={2025}
}

@inproceedings{moskvoretskii2025adaptive,
  title={Adaptive retrieval without self-knowledge? bringing uncertainty back home},
  author={Moskvoretskii, Viktor and Marina, Maria and Salnikov, Mikhail and Ivanov, Nikolay and Pletenev, Sergey and Galimzianova, Daria and Krayko, Nikita and Konovalov, Vasily and Nikishina, Irina and Panchenko, Alexander},
  booktitle={Proceedings of the 63rd Annual Meeting of the Association for Computational Linguistics (Volume 1: Long Papers)},
  pages={6355--6384},
  year={2025}
}

@inproceedings{su-etal-2024-api,
    title = "{API} Is Enough: Conformal Prediction for Large Language Models Without Logit-Access",
    author = "Su, Jiayuan  and
      Luo, Jing  and
      Wang, Hongwei  and
      Cheng, Lu",
    editor = "Al-Onaizan, Yaser  and
      Bansal, Mohit  and
      Chen, Yun-Nung",
    booktitle = "Findings of the Association for Computational Linguistics: EMNLP 2024",
    month = nov,
    year = "2024",
    address = "Miami, Florida, USA",
    publisher = "Association for Computational Linguistics",
    url = "https://aclanthology.org/2024.findings-emnlp.54/",
    doi = "10.18653/v1/2024.findings-emnlp.54",
    pages = "979--995"
}

@inproceedings{li-etal-2024-traq,
    title = "{TRAQ}: Trustworthy Retrieval Augmented Question Answering via Conformal Prediction",
    author = "Li, Shuo  and
      Park, Sangdon  and
      Lee, Insup  and
      Bastani, Osbert",
    editor = "Duh, Kevin  and
      Gomez, Helena  and
      Bethard, Steven",
    booktitle = "Proceedings of the 2024 Conference of the North American Chapter of the Association for Computational Linguistics: Human Language Technologies (Volume 1: Long Papers)",
    month = jun,
    year = "2024",
    address = "Mexico City, Mexico",
    publisher = "Association for Computational Linguistics",
    url = "https://aclanthology.org/2024.naacl-long.210/",
    doi = "10.18653/v1/2024.naacl-long.210",
    pages = "3799--3821"
}

@article{mohri2024language,
  title={Language models with conformal factuality guarantees},
  author={Mohri, Christopher and Hashimoto, Tatsunori},
  journal={arXiv preprint arXiv:2402.10978},
  year={2024}
}

@article{rubin2025conformal,
  title={Conformal Language Model Reasoning with Coherent Factuality},
  author={Rubin-Toles, Maxon and Gambhir, Maya and Ramji, Keshav and Roth, Aaron and Goel, Surbhi},
  journal={arXiv preprint arXiv:2505.17126},
  year={2025}
}

@misc{sheng2025analyzinguncertaintyllmasajudgeinterval,
      title={Analyzing Uncertainty of LLM-as-a-Judge: Interval Evaluations with Conformal Prediction}, 
      author={Huanxin Sheng and Xinyi Liu and Hangfeng He and Jieyu Zhao and Jian Kang},
      year={2025},
      eprint={2509.18658},
      archivePrefix={arXiv},
      primaryClass={cs.CL},
      url={https://arxiv.org/abs/2509.18658}, 
}

@article{kumar2023conformal,
  title={Conformal prediction with large language models for multi-choice question answering},
  author={Kumar, Bhawesh and Lu, Charlie and Gupta, Gauri and Palepu, Anil and Bellamy, David and Raskar, Ramesh and Beam, Andrew},
  journal={arXiv preprint arXiv:2305.18404},
  year={2023}
}

@inproceedings{ye2024benchmarking,
title={Benchmarking {LLM}s via Uncertainty Quantification},
author={Fanghua Ye and Mingming Yang and Jianhui Pang and Longyue Wang and Derek F. Wong and Emine Yilmaz and Shuming Shi and Zhaopeng Tu},
booktitle={The Thirty-eight Conference on NIPS Datasets and Benchmarks Track},
year={2024},
url={https://openreview.net/forum?id=L0oSfTroNE}
}

@article{team2024gemma,
  title={Gemma: Open models based on gemini research and technology},
  author={Team, Gemma and Mesnard, Thomas and Hardin, Cassidy and Dadashi, Robert and Bhupatiraju, Surya and Pathak, Shreya and Sifre, Laurent and Rivi{\`e}re, Morgane and Kale, Mihir Sanjay and Love, Juliette and others},
  journal={arXiv preprint arXiv:2403.08295},
  year={2024}
}

@article{shafer2008tutorial,
  title={A tutorial on conformal prediction.},
  author={Shafer, Glenn and Vovk, Vladimir},
  journal={Journal of Machine Learning Research},
  volume={9},
  number={3},
  year={2008}
}

@article{angelopoulos2021gentle,
  title={A gentle introduction to conformal prediction and distribution-free uncertainty quantification},
  author={Angelopoulos, Anastasios N and Bates, Stephen},
  journal={arXiv preprint arXiv:2107.07511},
  year={2021}
}

@article{campos2024conformal,
  title={Conformal prediction for natural language processing: A survey},
  author={Campos, Margarida and Farinhas, Ant{\'o}nio and Zerva, Chrysoula and Figueiredo, M{\'a}rio AT and Martins, Andr{\'e} FT},
  journal={Transactions of the Association for Computational Linguistics},
  volume={12},
  pages={1497--1516},
  year={2024},
  publisher={MIT Press 255 Main Street, 9th Floor, Cambridge, Massachusetts 02142, USA~…}
}

@inproceedings{yang2018hotpotqa,
  title={{HotpotQA}: A Dataset for Diverse, Explainable Multi-hop Question Answering},
  author={Yang, Zhilin and Qi, Peng and Zhang, Saizheng and Bengio, Yoshua and Cohen, William W. and Salakhutdinov, Ruslan and Manning, Christopher D.},
  booktitle={Conference on Empirical Methods in Natural Language Processing ({EMNLP})},
  year={2018}
}

@article{zhang2024conformal,
  title={Conformal structured prediction},
  author={Zhang, Botong and Li, Shuo and Bastani, Osbert},
  journal={arXiv preprint arXiv:2410.06296},
  year={2024}
}

@article{shinn2023reflexion,
  title={Reflexion: Language agents with verbal reinforcement learning},
  author={Shinn, Noah and Cassano, Federico and Gopinath, Ashwin and Narasimhan, Karthik and Yao, Shunyu},
  journal={Advances in neural information processing systems},
  volume={36},
  pages={8634--8652},
  year={2023}
}

@article{schuster2022confident,
  title={Confident adaptive language modeling},
  author={Schuster, Tal and Fisch, Adam and Gupta, Jai and Dehghani, Mostafa and Bahri, Dara and Tran, Vinh and Tay, Yi and Metzler, Donald},
  journal={Advances in Neural Information Processing Systems},
  volume={35},
  pages={17456--17472},
  year={2022}
}

@article{elasticsearch2018elasticsearch,
  title={Elasticsearch},
  author={Elasticsearch, BV},
  journal={software], version},
  volume={6},
  number={1},
  year={2018}
}

@inproceedings{yao2022react,
  title={React: Synergizing reasoning and acting in language models},
  author={Yao, Shunyu and Zhao, Jeffrey and Yu, Dian and Du, Nan and Shafran, Izhak and Narasimhan, Karthik R and Cao, Yuan},
  booktitle={The eleventh international conference on learning representations},
  year={2022}
}

@article{chakraborty2025principled,
  title={Principled Context Engineering for RAG: Statistical Guarantees via Conformal Prediction},
  author={Chakraborty, Debashish and Yang, Eugene and Khashabi, Daniel and Lawrie, Dawn and Duh, Kevin},
  journal={arXiv preprint arXiv:2511.17908},
  year={2025}
}

@article{kwiatkowski2019natural,
  title={Natural Questions: A Benchmark for Question Answering Research},
  author={Kwiatkowski, Tom and Palomaki, Jennimaria and Redfield, Olivia and Collins, Michael and Parikh, Ankur and Alberti, Chris and Epstein, Danielle and Polosukhin, Illia and Devlin, Jacob and Lee, Kenton and Toutanova, Kristina and Jones, Llion and Kelcey, Matthew and Chang, Ming-Wei and Dai, Andrew M. and Uszkoreit, Jakob and Le, Quoc and Petrov, Slav},
  journal={Transactions of the Association for Computational Linguistics},
  volume={7},
  pages={453--466},
  year={2019}
}

@inproceedings{joshi2017triviaqa,
  title={{TriviaQA}: A Large Scale Distantly Supervised Challenge Dataset for Reading Comprehension},
  author={Joshi, Mandar and Choi, Eunsol and Weld, Daniel S. and Zettlemoyer, Luke},
  booktitle={Proceedings of the 55th Annual Meeting of the Association for Computational Linguistics (Volume 1: Long Papers)},
  pages={1601--1611},
  year={2017},
  address={Vancouver, Canada}
}

@article{trivedi2022musique,
  title={{MuSiQue}: Multihop Questions via Single-hop Question Composition},
  author={Trivedi, Harsh and Balasubramanian, Niranjan and Khot, Tushar and Sabharwal, Ashish},
  journal={Transactions of the Association for Computational Linguistics},
  volume={10},
  pages={539--554},
  year={2022}
}

@inproceedings{ho2020constructing,
  title={Constructing A Multi-Hop {QA} Dataset for Comprehensive Evaluation of Reasoning Steps},
  author={Ho, Xanh and Duong Nguyen, Anh-Khoa and Sugawara, Saku and Aizawa, Akiko},
  booktitle={Proceedings of the 28th International Conference on Computational Linguistics},
  pages={6609--6625},
  year={2020},
  address={Barcelona, Spain (Online)},
  publisher={International Committee on Computational Linguistics}
}

@misc{qwen3.5,
    title  = {{Qwen3.5}: Towards Native Multimodal Agents},
    author = {{Qwen Team}},
    month  = {February},
    year   = {2026},
    url    = {https://qwen.ai/blog?id=qwen3.5}
}

@article{quach2023conformal,
  title={Conformal language modeling},
  author={Quach, Victor and Fisch, Adam and Schuster, Tal and Yala, Adam and Sohn, Jae Ho and Jaakkola, Tommi S and Barzilay, Regina},
  journal={arXiv preprint arXiv:2306.10193},
  year={2023}
}

@article{zhou2025conformal,
  title={Conformal prediction: A data perspective},
  author={Zhou, Xiaofan and Chen, Baiting and Gui, Yu and Cheng, Lu},
  journal={ACM Computing Surveys},
  year={2025},
  publisher={ACM New York, NY}
}

@article{achiam2023gpt,
  title={Gpt-4 technical report},
  author={Achiam, Josh and Adler, Steven and Agarwal, Sandhini and Ahmad, Lama and Akkaya, Ilge and Aleman, Florencia Leoni and Almeida, Diogo and Altenschmidt, Janko and Altman, Sam and Anadkat, Shyamal and others},
  journal={arXiv preprint arXiv:2303.08774},
  year={2023}
}

@article{rane2023contribution,
  title={Contribution and performance of ChatGPT and other Large Language Models (LLM) for scientific and research advancements: a double-edged sword},
  author={Rane, Nitin Liladhar and Tawde, Abhijeet and Choudhary, Saurabh P and Rane, Jayesh},
  journal={International Research Journal of Modernization in Engineering Technology and Science},
  volume={5},
  number={10},
  pages={875--899},
  year={2023}
}

@article{wang2025large,
  title={Large language models for robotics: Opportunities, challenges, and perspectives},
  author={Wang, Jiaqi and Shi, Enze and Hu, Huawen and Ma, Chong and Liu, Yiheng and Wang, Xuhui and Yao, Yincheng and Liu, Xuan and Ge, Bao and Zhang, Shu},
  journal={Journal of Automation and Intelligence},
  volume={4},
  number={1},
  pages={52--64},
  year={2025},
  publisher={Elsevier}
}

@misc{li2026retrievalgenerationunifiedframework,
      title={Retrieval as Generation: A Unified Framework with Self-Triggered Information Planning}, 
      author={Bo Li and Mingda Wang and Gexiang Fang and Shikun Zhang and Wei Ye},
      year={2026},
      eprint={2604.11407},
      archivePrefix={arXiv},
      primaryClass={cs.CL},
      url={https://arxiv.org/abs/2604.11407}, 
}

@inproceedings{DBLP:conf/aaai/LiTXCZY26,
  author       = {Bo Li and
                  Tian Tian and
                  Zhenghua Xu and
                  Hao Cheng and
                  Shikun Zhang and
                  Wei Ye},
  title        = {Modeling Uncertainty Trends for Timely Retrieval in Dynamic {RAG}},
  booktitle    = {Fortieth {AAAI} Conference on Artificial Intelligence, Thirty-Eighth
                  Conference on Innovative Applications of Artificial Intelligence,
                  Sixteenth Symposium on Educational Advances in Artificial Intelligence,
                  {AAAI} 2026, Singapore, January 20-27, 2026},
  pages        = {31527--31535},
  publisher    = {{AAAI} Press},
  year         = {2026},
}
\bibliographystyle{colm2026_conference}

\appendix
\section{Appendix}
\subsection{Experiment Setting}
\label{setting}
For adaptive RAG and ReAct, retrieval is performed using an Elasticsearch-based retriever~\citep{elasticsearch2018elasticsearch} over Wikipedia, returning the top-$K$ passages at each turn with $K=10$. We set the maximum number of turns to $T=3$, draw $M=15$ stochastic answer samples per turn using temperature $1.4$, and use $G=20$ grid steps per dimension when optimizing the allocation of turn-specific error budgets. We evaluate target error rates $\alpha \in \{0.05, 0.10, 0.15, 0.20, 0.25\}$ and report results averaged over three random seeds for the dataset split. In Eq.~\ref{nepenal}, we set $\eta = 0.1$, and throughout all experiments we fix the retrieval error budget to $\alpha_{\mathrm{ret}} = 0.1$.
\subsection{Proof of Retrieval Calibration Coverage Guarantee}

\begin{theorem}[Retrieval Coverage Guarantee]
Let $\{(x_i, \mathcal{P}_i^*)\}_{i=1}^{n_{\mathrm{cal}}}$ be the calibration set, where $\mathcal{P}_i^* = \{g \in \mathcal{G}_i \mid \mathcal{T}(g, x_i) \neq \emptyset\}$ is the set of retrievable gold passages for question $x_i$. Let $s^*(g, x_i) = \max_{t \in \mathcal{T}(g, x_i)} s_t(g)$ be the optimistic relevance score. Collect all such scores into $\mathcal{S}_{\mathrm{ret}} = \{s^*(g, x_i) \mid g \in \mathcal{P}_i^*, \, i = 1, \ldots, n_{\mathrm{cal}}\}$ with $|\mathcal{S}_{\mathrm{ret}}| = n$. Define the conformal threshold as:
\begin{equation}
\hat{q}_{\mathrm{ret}} = \mathrm{Quantile}\!\left(\mathcal{S}_{\mathrm{ret}};\, \frac{(n-1)\alpha}{n}\right).
\end{equation}
Then for a new exchangeable test example, any retrievable gold passage $g \in \mathcal{P}_{n+1}^*$ satisfies:
\begin{equation}
\mathbb{P}\!\left(s^*(g, x_{n+1}) \geq \hat{q}_{\mathrm{ret}}\right) \geq 1 - \alpha.
\end{equation}
\end{theorem}

\begin{proof}
Under the exchangeability assumption, the $n$ calibration scores and the test score $s_{n+1} = s^*(g, x_{n+1})$ form an exchangeable sequence of size $n+1$. By the split conformal prediction guarantee, setting the threshold at the $\frac{(n-1)\alpha}{n}$-quantile of the $n$ calibration scores yields:
\begin{equation}
\mathbb{P}\!\left(s_{n+1} < \hat{q}_{\mathrm{ret}}\right) \leq \frac{(n-1)\alpha}{n}.
\end{equation}
Since $\frac{(n-1)\alpha}{n} < \alpha$ for all $n \geq 1$, we have:
\begin{equation}
\mathbb{P}\!\left(s^*(g, x_{n+1}) \geq \hat{q}_{\mathrm{ret}}\right) \geq 1 - \frac{(n-1)\alpha}{n} > 1 - \alpha.
\end{equation}

Since this holds marginally for each retrievable gold passage, the expected gold retention rate satisfies:
\begin{equation}
\mathbb{E}\!\left[\frac{|\{g \in \mathcal{P}_{n+1}^* \mid \exists\, t : s_t(g) \geq \hat{q}_{\mathrm{ret}}\}|}{|\mathcal{P}_{n+1}^*|}\right] \geq 1 - \frac{(n-1)\alpha}{n} \geq 1 - \alpha.
\end{equation}
\end{proof}

\begin{remark}
We calibrate on $\mathcal{P}_i^*$ (retrievable gold passages) rather than $\mathcal{G}_i$ (all annotated gold passages), avoiding inflation of the threshold with unretrievable passages.
\end{remark}


\subsection{Proof of Prediction Set Coverage Guarantee}

\begin{theorem}[Prediction Set Coverage Guarantee]
Let $\{(x_i, y_i, \{A_t^{i}\}_{t=0}^{T})\}_{i=1}^{n_{\mathrm{cal}}}$ be the calibration set. Define the answerable subset $\mathcal{D}_{\mathrm{cal}}^{\mathrm{ans}} = \{i \mid y_i \in \bigcup_{t=0}^{T} A_t^{i}\}$ with $|\mathcal{D}_{\mathrm{cal}}^{\mathrm{ans}}| = m$. For each $i \in \mathcal{D}_{\mathrm{cal}}^{\mathrm{ans}}$, let $f(C_{\mathrm{gold}}^{i})$ be the maximum frequency score of the gold-containing cluster (Eq.~\ref{eq:freq_score}). Define:
\begin{equation}
\hat{q}_{\mathrm{freq}} = \mathrm{Quantile}\!\left(\{f(C_{\mathrm{gold}}^{i})\}_{i \in \mathcal{D}_{\mathrm{cal}}^{\mathrm{ans}}};\, \frac{(m-1)\alpha}{m}\right).
\end{equation}
Then for a new exchangeable test example $(x_{n+1}, y_{n+1})$, the prediction set $C(x_{n+1})$ satisfies:
\begin{equation}
\mathbb{P}\!\left(y_{n+1} \in C(x_{n+1}) \text{ if } y_{n+1} \in \textstyle\bigcup_{t} A_t^{n+1)},\;\; \text{``Can't Answer''} \in C(x_{n+1}) \text{ if } y_{n+1} \notin \textstyle\bigcup_{t} A_t^{n+1)}\right) \geq 1 - \alpha.
\end{equation}
\end{theorem}

\begin{proof}
We analyze the two cases separately.

\textbf{Case 1: The gold answer is never sampled} ($y_{n+1} \notin \bigcup_{t=0}^{T} A_t^{n+1)}$).

The NE thresholds $\{\hat{q}_t\}$ are calibrated on the unanswerable subset $U_t$ at each turn (Eq.~\ref{NE}), with per-turn error budgets $\{\alpha_t\}$ satisfying:
\begin{equation}
\sum_{t=0}^{T} (1 - c_{\mathrm{ans}}^t) \cdot \alpha_t \leq (1 - c_{\mathrm{ans}}^{\mathrm{final}}) \cdot \alpha.
\end{equation}
This bounds the total probability of incorrectly early-stopping an unanswerable question before it reaches the final turn. When a question is not early-stopped and exhausts all $T$ turns, the ``Can't Answer'' label is always appended by construction (Eq.~\ref{eq:freq_score}). Therefore:
\begin{equation}
\mathbb{P}\!\left(\text{``Can't Answer''} \in C(x_{n+1}) \mid y_{n+1} \notin \textstyle\bigcup_t A_t^{n+1)}\right) \geq 1 - \alpha.
\end{equation}

\textbf{Case 2: The gold answer appears in some turn's samples} ($y_{n+1} \in \bigcup_{t=0}^{T} A_t^{n+1)}$).

By exchangeability, the $m$ calibration scores $\{f(C_{\mathrm{gold}}^{i})\}_{i \in \mathcal{D}_{\mathrm{cal}}^{\mathrm{ans}}}$ and the test score $f(C_{\mathrm{gold}}^{n+1)})$ form an exchangeable sequence of size $m+1$. By the split conformal guarantee at quantile level $\frac{(m-1)\alpha}{m}$:
\begin{equation}
\mathbb{P}\!\left(f(C_{\mathrm{gold}}^{n+1)}) < \hat{q}_{\mathrm{freq}}\right) \leq \frac{(m-1)\alpha}{m} < \alpha.
\end{equation}
Therefore:
\begin{equation}
\mathbb{P}\!\left(y_{n+1} \in C(x_{n+1}) \mid y_{n+1} \in \textstyle\bigcup_t A_t^{n+1)}\right) \geq 1 - \frac{(m-1)\alpha}{m} > 1 - \alpha.
\end{equation}

\textbf{Combining both cases.} Let $E$ denote the event $y_{n+1} \in \bigcup_t A_t^{n+1)}$. By the law of total probability:
\begin{align*}
\mathbb{P}\!\left(\text{coverage holds}\right) &= \mathbb{P}(E) \cdot \mathbb{P}\!\left(y_{n+1} \in C(x_{n+1}) \mid E\right) + \mathbb{P}(\bar{E}) \cdot \mathbb{P}\!\left(\text{``Can't Answer''} \in C(x_{n+1}) \mid \bar{E}\right) \\
&\geq \mathbb{P}(E) \cdot (1 - \alpha) + \mathbb{P}(\bar{E}) \cdot (1 - \alpha) \\
&= 1 - \alpha.
\end{align*}
\end{proof}
\begin{remark}[Role of the error budget decomposition]
The NE thresholds $\{\hat{q}_t\}_{t=0}^{T}$ control which turn each question stops at. The error budget constraint (Eq.~\ref{eq:error_decomposition}) ensures that unanswerable questions are not prematurely early-stopped, preserving the ``Can't Answer'' mechanism, while allowing efficient per-hop early stopping without violating the overall $1 - \alpha$ coverage guarantee.
\end{remark}

\begin{remark}[Conditional vs.\ marginal coverage]
The guarantees above are marginal, holding on average over the randomness of calibration and test data. Conditional coverage for specific subgroups would require group-conditional calibration, which we leave for future work.
\end{remark}

\end{document}